\documentclass{article} 
\usepackage{iclr2025_conference,times}

\usepackage{hyperref}
\usepackage{url}

\usepackage[ruled,vlined]{algorithm2e}
\usepackage{amsfonts}

\usepackage{amsmath}
\usepackage{amssymb}
\usepackage{mathtools}
\usepackage{amsthm}

\usepackage[capitalize,noabbrev]{cleveref}

\theoremstyle{plain}
\newtheorem{theorem}{Theorem}[section]
\newtheorem{proposition}[theorem]{Proposition}

\theoremstyle{definition}

\theoremstyle{remark}

\usepackage{CJKutf8}
\usepackage{wrapfig}
\usepackage{lipsum}

\usepackage{booktabs} 
\usepackage{multirow}
\usepackage{colortbl}
\usepackage{graphicx}
\usepackage{subcaption}

\usepackage{amsmath,amsfonts,bm}

\def\eqref#1{equation~\ref{#1}}

\def\1{\bm{1}}

\def\vh{{\bm{h}}}

\def\vy{{\bm{y}}}

\def\mH{{\bm{H}}}

\def\mK{{\bm{K}}}

\def\mX{{\bm{X}}}
\def\mY{{\bm{Y}}}

\DeclareMathAlphabet{\mathsfit}{\encodingdefault}{\sfdefault}{m}{sl}
\SetMathAlphabet{\mathsfit}{bold}{\encodingdefault}{\sfdefault}{bx}{n}

\def\gG{{\mathcal{G}}}

\def\gN{{\mathcal{N}}}
\def\gO{{\mathcal{O}}}

\def\gS{{\mathcal{S}}}

\def\sA{{\mathbb{A}}}

\def\sK{{\mathbb{K}}}

\def\sM{{\mathbb{M}}}

\def\sR{{\mathbb{R}}}
\def\sS{{\mathbb{S}}}

\def\sV{{\mathbb{V}}}

\newcommand\db[1]{\textcolor{darkblue}{#1}}
\newcommand\dr[1]{\textcolor{darkred}{#1}}

\newcommand{\where}[0]{\text{where}}

\newcommand{\WLS}[0]{\mathtt{WLS}}

\newcommand{\gnn}[0]{\texttt{GNN}}

\newcommand{\vphi}[0]{\bm{\phi}}

\definecolor{weakgray}{HTML}{e7e7e7}
\definecolor{weakblue}{HTML}{E3F2FD}
\definecolor{weakred}{HTML}{FFEBEE}
\definecolor{midblue}{HTML}{bbdefb}
\definecolor{midred}{HTML}{ffcdd2}

\definecolor{darkblue}{HTML}{1565c0}
\definecolor{darkred}{HTML}{c62828}

\newcommand{\WLKSZD}[0]{WLKS-$\{0,D\}$}

\newcommand{\sub}[0]{\text{sub}}

\newcommand{\PPIBP}[0]{\textsf{\fontsize{8.5pt}{8.5pt}\selectfont PPI-BP}}
\newcommand{\HPOMetab}[0]{\textsf{\fontsize{8.5pt}{8.5pt}\selectfont HPO-Metab}}
\newcommand{\HPONeuro}[0]{\textsf{\fontsize{8.5pt}{8.5pt}\selectfont HPO-Neuro}}
\newcommand{\EMUser}[0]{\textsf{\fontsize{8.5pt}{8.5pt}\selectfont EM-User}}

\newcommand{\PPIBPb}[0]{\textsf{\fontsize{8.5pt}{8.5pt}\selectfont PPI-BP} }

\newcommand{\HPONeurob}[0]{\textsf{\fontsize{8.5pt}{8.5pt}\selectfont HPO-Neuro} }

\newcommand{\Density}[0]{\textsf{\fontsize{8.5pt}{8.5pt}\selectfont Density}}
\newcommand{\CutRatio}[0]{\textsf{\fontsize{8.5pt}{8.5pt}\selectfont Cut-Ratio}}
\newcommand{\Coreness}[0]{\textsf{\fontsize{8.5pt}{8.5pt}\selectfont Coreness}}
\newcommand{\Component}[0]{\textsf{\fontsize{8.5pt}{8.5pt}\selectfont Component}}

\newcommand{\CutRatiob}[0]{\textsf{\fontsize{8.5pt}{8.5pt}\selectfont Cut-Ratio} }
\newcommand{\Corenessb}[0]{\textsf{\fontsize{8.5pt}{8.5pt}\selectfont Coreness} }

\title{Generalizing Weisfeiler-Lehman Kernels \\ to Subgraphs}

\author{Dongkwan Kim \& Alice Oh \\
KAIST, Republic of Korea \\
\texttt{dongkwan.kim@kaist.ac.kr, alice.oh@kaist.edu} \\
}

\iclrfinalcopy 
\begin{document}

\maketitle

\begin{abstract}
Subgraph representation learning has been effective in solving various real-world problems. However, current graph neural networks (GNNs) produce suboptimal results for subgraph-level tasks due to their inability to capture complex interactions within and between subgraphs. To provide a more expressive and efficient alternative, we propose WLKS, a Weisfeiler-Lehman (WL) kernel generalized for subgraphs by applying the WL algorithm on induced $k$-hop neighborhoods. We combine kernels across different $k$-hop levels to capture richer structural information that is not fully encoded in existing models. Our approach can balance expressiveness and efficiency by eliminating the need for neighborhood sampling. In experiments on eight real-world and synthetic benchmarks, WLKS significantly outperforms leading approaches on five datasets while reducing training time, ranging from 0.01x to 0.25x compared to the state-of-the-art.
\end{abstract}

\section{Introduction}
Subgraph representation learning has effectively tackled various real-world problems~\citep{bordes2014question,luo2022shine,hamidi2022subgraph,maheshwari2024timegraphs}. However, existing graph neural networks (GNNs) still produce suboptimal representations for subgraph-level tasks since they fail to capture arbitrary interactions between and within subgraph structures. These GNNs cannot capture high-order interactions beyond and even in their receptive fields. Thus, state-of-the-art models for subgraphs have to employ hand-crafted channels~\citep{alsentzer2020subgraph}, node labeling~\citep{ wang2022glass}, and structure approximations~\citep{kim2024translating} to encode subgraphs' complex internal and border structures.

As an elegant and efficient alternative, we generalize graph kernels to subgraphs, which measure the structural similarity between pairs of graphs. We propose WLKS, the Weisfeiler-Lehman (WL) Kernel for Subgraphs based on WL graph kernel~\citep{shervashidze2009fast}. Specifically, we apply the WL algorithm~\citep{leman1968reduction} on induced $k$-hop subgraphs around the target subgraph for all possible $k$s. The WL algorithm's output (i.e., the color histogram) for each $k$ encodes structures in the receptive field of the $k$-layer GNNs; thus, the corresponding kernel matrix can represent the similarity of $k$-hop subgraph pairs. A classifier using this kernel can be trained without GPUs in a computationally efficient way compared to deep GNNs.

To enhance the expressive power, we linearly combine kernel matrices of different $k$-hops. The motivation is that simply using larger hops for WL histograms does not necessarily lead to more expressive representations. We theoretically demonstrate that WL histograms of the $(k+1)$-hop are not strictly more expressive than those of $k$-hop in distinguishing isomorphic structures, while $(k+1)$-hop structures include entire $k$-hop structures. Therefore, combining kernel matrices across multiple $k$-hop levels can capture richer structural information around subgraphs.

However, sampling $k$-hop subgraphs can increase the time and space complexity, as the number of nodes in the $k$-hop neighborhoods grows exponentially~\citep{hamilton2017inductive}. To mitigate this issue, we choose only two values of $k$: 0 and the diameter $D$ of the global graph. No neighborhood sampling is required for the case where $k = 0$ since it only uses the internal structure. When $k$ is set to the diameter $D$, the expanded subgraph encompasses the entire global graph, making the $k$-hop neighborhood identical for all subgraphs. Consequently, there is no need for explicit neighborhood sampling in this case; we only perform the WL algorithm on the global graph once. This approach balances expressiveness and efficiency, providing a practical solution for subgraph-level tasks.

We evaluate WLKS's classification performance and efficiency with four real-world and four synthetic benchmarks~\citep{alsentzer2020subgraph}. Our model outperforms the best-performed methods across five of the eight datasets. Remarkably, this performance is achieved with $\times 0.01$ to $\times 0.53$ training time compared to the state-of-the-art models. Moreover, unlike existing models, WLKS does not require pre-computation, pre-training embeddings, utilizing GPUs, and searching a large hyperparameter space.

The main contributions of our paper are summarized as follows. First, we propose WLKS, a generalization of graph kernels to subgraphs. Second, we theoretically show that combining WLKS matrices from multiple $k$-hop neighborhoods can increase the expressiveness. Third, we evaluate our method on real-world and synthetic benchmarks and demonstrate superior performance in a significantly efficient way. We make our
code available for future research\footnote{\href{https://github.com/dongkwan-kim/WLKS}{https://github.com/dongkwan-kim/WLKS}}.

\section{Related Work}
WLKS is a `graph kernel' method designed for `subgraph representation learning.' This section explains both of these areas and their relationship to our model.

\paragraph{Subgraph Representation Learning}

Subgraph representation learning can address various real-world challenges by capturing higher-order interactions that nodes, edges, or entire graphs cannot model. For example, subgraphs can formulate diseases and patients in gene networks~\citep{luo2022shine}, teams in collaboration networks~\citep{hamidi2022subgraph}, and communities in mobile game user networks~\citep{zhang2023constrained}. Existing methods are often domain-specific~\citep{zhang2023constrained,li2023self,trumper2023performance,ouyang2024bitcoin,maheshwari2024timegraphs} or rely on strong assumptions about the subgraph~\citep{meng2018subgraph,hamidi2022subgraph,kim2022models,luo2022shine,liu2023position}, limiting their generalizability.

Recent deep graph neural networks designed for subgraph-level tasks can apply to any subgraph type without specific assumptions. However, they often generate suboptimal representations due to their inability to capture arbitrary interactions between and within subgraph structures. They struggle to account for high-order interactions beyond their limited receptive fields; thus, they should incorporate additional techniques, including hand-crafted channels~\citep{alsentzer2020subgraph}, node labeling~\citep{wang2022glass}, random-walk sampling~\citep{jacob2023stochastic}, and structural approximations~\citep{kim2024translating}. In contrast, we design kernels that can capture local and global interactions of subgraphs, respectively, to enable simple but strong subgraph prediction. We formally compare our proposed WLKS with representative prior models in Appendix~\ref{appendix:comparison}.

\paragraph{Graph Kernels}

Graph kernels are algorithms to measure the similarity between graphs to enable the kernel methods, such as Support Vector Machines (SVMs) to graph-structured data~\citep{vishwanathan2010graph}. Early examples measure the graph similarity based on random walks~\citep{kashima2003marginalized} or shorted paths~\citep{borgwardt2005shortest}. One of the most influential graph kernels is the Weisfeiler-Lehman (WL) kernel~\citep{shervashidze2009fast}, which leverages the WL isomorphism test to refine node labels iteratively, improving the expressiveness of the graph structure comparison. While the WL test is designed for graph isomorphism, WL kernels capture structural similarities using the WL test's outcomes even when graphs are not strictly isomorphic (See Appendix~\ref{appendix:wl_test_and_kernel} for detailed comparison). Kernels for graph-level prediction by counting, matching, and embedding subgraphs have been deeply explored~\citep{shervashidze2009efficient, kriege2012subgraph, yanardag2015deep, narayanan2016subgraph2vec}. However, there has been no research on kernels to solve subgraph-level tasks by computing the similarity of subgraphs and their surroundings. To the best of our knowledge, our paper is the first to investigate this approach.

\section{WL Graph Kernels for Subgraph-Level Tasks}
This section introduces WLKS, the WL graph kernels generalized for subgraphs. We first describe the original WL algorithm and its extension for subgraphs, which is a foundation of WLKS. Then, we suggest WLKS and its enhancement of expressiveness and efficiency. Finally, we introduce how to integrate continuous features with WLKS models.

\subsection{Subgraph Representation Learning}

We first formalize subgraph representation learning as a classification task. Let $\gG = (\sV, \sA)$ represent a global graph, where $\sV$ denotes a set of nodes (with $|\sV| = N$) and $\sA \subset \sV \times \sV$ represents a set of edges (with $|\sA| = E$). A subgraph $\gS = (\sV^{\sub}, \sA^{\sub})$ is a graph formed by subsets of nodes and edges in the global graph $\gG$ (with $|\sV^{\sub}| = N^{\sub}$ and $|\sA^{\sub}| = E^{\sub}$). There exists a set of $M$ subgraphs, with $M < N$, denoted as $\sS = \{ \gS_1, \gS_2, \dots, \gS_M \}$. In a subgraph classification task, the model learns representation $\vh_i \in \sR^{F}$ and the logit vector $\vy_i \in \sR^{C}$ for $\gS_i$ where $F$ and $C$ are the dimension size and the number of classes, respectively.

\subsection{1-WL Algorithm for $k$-hop Subgraphs}

\paragraph{1-WL for Graphs}

We briefly introduce the 1-dimensional Weisfeiler-Lehman (1-WL) algorithm. As illustrated in Algorithm~\ref{alg:wl}, the 1-WL is an iterative node-color refinement by updating node colors based on a multiset of neighboring node colors. This process produces a histogram of refined coloring that captures graph structure, which can distinguish non-isomorphic graphs in the WL isomorphism test.

\begin{algorithm}[H]\label{alg:wl}
\DontPrintSemicolon
\SetAlgoLined
\vspace{0.1cm}
\KwIn{Graph $\gG = (\sV, \sA)$ and $T$ iterations}
\KwOut{Refined node coloring $( c_{1}^{T}, c_{2}^{T}, ..., c_{|\sV|}^{T} )$ for nodes in $\sV$ after $T$ iterations} \vspace{0.1cm}
Initialize $c_{v}^{0}$ for all $v \in \sV$\;
\For{$i \leftarrow 1$ \KwTo $T$}{
    \For{node $v \in \sV$}{
        $\sM_v \leftarrow$ multiset of labels $\{ c_{u}^{i-1} \mid u \in \gN(v) \}$\;
        $\Tilde{c}_{v}^{i} \leftarrow$ concatenate $c_{v}^{i-1}$ and sorted $\sM_{v}$\;
    }
    Use a bijective function to map each unique $\Tilde{c}_{v}^{i}$ to a new color $c_{v}^{i}$\;
}
\Return{$ ( c_{1}^{T}, c_{2}^{T}, ..., c_{|\sV|}^{T} ) $}
\caption{1-WL Algorithm}
\vspace{0.1cm}
\end{algorithm}

\paragraph{1-WL for Subgraphs (WLS)}

We then propose the WLS, the 1-WL algorithm generalized for subgraphs. Since surrounding structures are the core difference between graphs and subgraphs, the main contribution of the WLS lies in encoding the $k$-hop neighborhoods of the subgraph. Here, $k$ will be denoted in superscript as $\WLS^k$ if a specific $k$ is given.

Formally, for a subgraph $\gS = (\sV^{\sub}, \sA^{\sub})$ in a global graph $\gG = (\sV, \sA)$, the $\WLS^{k}$'s goal is to get the refined colors of nodes in $\sV^{\sub}$, where each color represents a unique subtree in $k$-hop neighborhoods. As in the Algorithm~\ref{alg:wlks}, we first extract the $k$-hop subgraph $\gS^{k}$ of $\gS$, which contains all nodes in $\gS$ as well as any nodes in $\gG$ that are reachable from the nodes in $\gS$ within $k$ hops. The 1-WL algorithm is then run on this induced $k$-hop subgraph to generate the colors of the nodes in $\gS^k$. The WLS returns the node coloring belonging to the original $\gS$, not in $\gS^k$. In general, $k$-hop neighborhoods are much larger than the original subgraph, so using all the colors in $\gS^k$ will likely produce a coloring irrelevant to the target subgraph.

\begin{algorithm}[H]\label{alg:wlks}
\DontPrintSemicolon
\SetAlgoLined
\vspace{0.1cm}
\KwIn{A subgraph $\gS = (\sV^{\sub}, \sA^{\sub})$, a global graph $\gG = (\sV, \sA)$, and $T$ iterations}
\KwOut{Refined node coloring $( c_{1}^{T}, c_{2}^{T}, ..., c_{|\sV^{\sub}|}^{T} )$ for nodes in $\sV^{\sub}$ after $T$ iterations} \vspace{0.2cm}
Sample $\gS^{k} = (\sV^{\sub,k}, \sA^{\sub,k})$, which is the induced $k$-hop subgraph of $\gG$ around all nodes in $\gS$ reachable within $k$ hops\;
Run  1-WL (Algorithm \ref{alg:wl}) on $(\gS^{k}, T)$ to get node colors in $\sV^{\sub, k}$ \;
\Return{ $( c_{1}^{T}, c_{2}^{T}, ..., c_{|\sV^{\sub}|}^{T} )$. Note that this coloring is about nodes in $\gS$, not $\gS^k$.}
\caption{$\WLS^{k}$ Algorithm: 1-WL for subgraphs with their $k$-hop neighborhoods}
\vspace{0.1cm}
\end{algorithm}

After $\WLS^k$'s color refinement, we can get a feature vector (or a color histogram) $\vphi_{\gS}^{k} \in \sR^{\#\text{colors}}$ which is the aggregation of the refined colors in $\gS$. Each element $\vphi_{\gS}^{k}[c]$ is the number of occurrences of the color $c$ in the output of $\WLS^{k}$. We illustrate an example of a subgraph and its $\WLS^k$ outputs for different $k$s in Figure~\ref{fig:model}.

\paragraph{WL Kernels for Subgraphs (WLKS)}

Now, we suggest WLKS, the corresponding kernel matrix $\mK_{\WLS}^{k} \in \sR^{M \times M}$ of which is defined as the number of common subtree patterns of two subgraphs in their $k$-hop neighborhoods. That is, each element can be formulated as an inner product of a pair of $\vphi_{\ast}^{k}$. This WLKS is a valid kernel since $\mK_{\WLS}^{k}$ is positive semi-definite for all non-negative $k$s, as demonstrated in Proposition~\ref{prop:wls_k_psd}.


\begin{proposition}\label{prop:wls_k_psd}
\vspace{0.2cm}
$\forall k \geq 0, \mK_{\WLS}^{k}$ is positive semi-definite (p.s.d.).
\end{proposition}
\begin{proof} \vspace{-0.3cm}
Each element in $\mK_{\WLS}^{k}$ is defined as an inner product of two feature vectors $\vphi_{\ast}^{k}$. This leads $\sum_{i=1}^{M} \sum_{j=1}^{M} c_i c_j \langle \vphi_{i}^{k}, \vphi_{j}^{k} \rangle = \langle \sum_{i=1}^{M} c_i \vphi_{i}^{k}, \sum_{j=1}^{M} c_j \vphi_{j}^{k} \rangle = \| \sum_{i=1}^{M} c_i \vphi_{i}^{k} \|^2 \geq 0$ for any real $c$. Thus, $\mK_{\WLS}^{k}$ is positive semi-definite.
\end{proof}

\subsection{Expressiveness difference of the WLS between $k$ and $k+1$}\label{sec:expressivness}

\begin{figure}
  \centering
  \includegraphics[width=\textwidth]{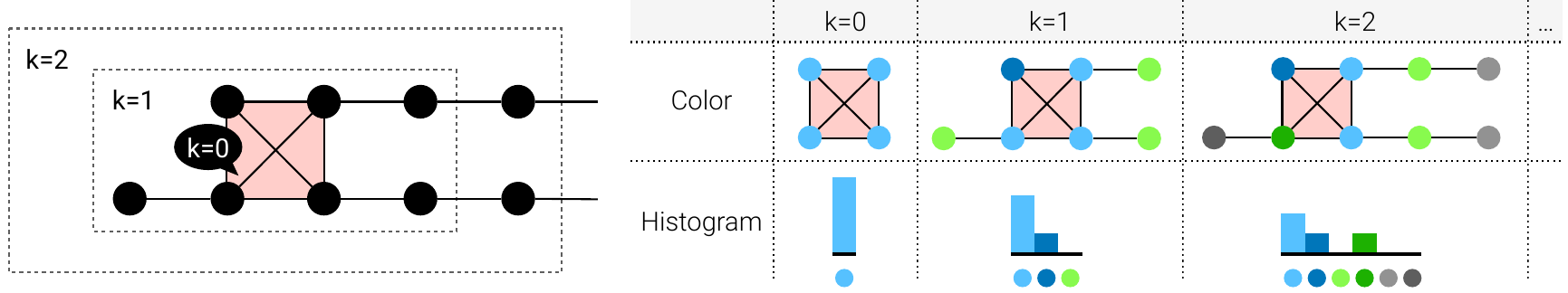}
  \vspace{-0.4cm}
  \caption{An example of $\WLS^k$ algorithm  (Algorithm~\ref{alg:wlks}) for $k \in \{ 0, 1, 2 \}$. \textbf{Left:} A subgraph (red shade) and its $k$-hop neighborhoods (dashed lines). \textbf{Right:} The outputs of $\WLS^k$ algorithm as colors and histograms for the left subgraph. We visualize each iteration of the algorithm in Appendix~\ref{appendix:model_steps}. The WLKS kernel matrix for each $k$ is constructed by an inner product of histogram pairs.}
  \vspace{-0.2cm}
  \label{fig:model}
\end{figure}
How do we choose $k$? Intuitively, selecting one large $k$ seems reasonable since the $k$-hop neighborhoods include the $k'$-hop structures of all smaller $k'$s. Against this intuition, we present a theoretical analysis that the $\WLS^{k+1}$ histogram is not strictly more expressive than the $\WLS^{k}$ histogram.

In Proposition~\ref{prop:wls_main}, we show that non-equivalent colorings of two subgraphs in $\WLS^{k+1}$ do not guarantee non-equivalent colorings in $\WLS^{k}$. This is also true for the inverse: equivalent colorings in $\WLS^{k+1}$ do not guarantee equivalent colorings in $\WLS^{k}$. We obtain the same conclusion as Proposition~\ref{prop:wls_main} for GNNs as powerful as the WL test (e.g., \citet{wang2022glass}), and some recent models are based on even less powerful GNNs than the WL test (e.g., \citet{kim2024translating}).

\begin{proposition}\label{prop:wls_main}
\vspace{0.1cm}
Given two subgraphs $\gS_1$ and $\gS_2$ of a global graph $\gG$ and $T$ iterations,
\begin{align}
\label{eq:wls_main_1} \WLS^{\db{k+1}} (\gS_1) \not\equiv \WLS^{\db{k+1}} (\gS_2) \nRightarrow \WLS^{\dr{k}} (\gS_1) \not\equiv \WLS^{\dr{k}} (\gS_2), \\ 
\label{eq:wls_main_2} \WLS^{\db{k+1}} (\gS_1) \equiv \WLS^{\db{k+1}} (\gS_2) \nRightarrow \WLS^{\dr{k}} (\gS_1) \equiv \WLS^{\dr{k}} (\gS_2),
\end{align}
for any $k < T$ where $\WLS^{k} (\gS) := \WLS^{k} (\gS, \gG, T)$ and `$\equiv$' denotes the equivalence of colorings.
\end{proposition}

\begin{proof} \vspace{-0.2cm}
We will prove both statements by contradiction.
\paragraph{Proof of Equation~\ref{eq:wls_main_1}} \vspace{-0.3cm}
For the sake of contradiction, assume that whenever \(\WLS^{\db{k+1}} (\gS_1) \not\equiv \WLS^{\db{k+1}} (\gS_2)\), it must follow that \(\WLS^{\dr{k}} (\gS_1) \not\equiv \WLS^{\dr{k}} (\gS_2)\). Consider two subgraphs \(\gS_1\) and \(\gS_2\) of a global graph \(\gG\) such that their \(k\)-hop neighborhoods are isomorphic, i.e., \(\gS_1^{k} \equiv \gS_2^{k}\), but their \((k+1)\)-hop neighborhoods have non-identical subtree patterns of height-$T$ rooted at subgraphs. That is, within the \(k\)-hop radius, \(\gS_1\) and \(\gS_2\) have identical structures, but beyond that, their structures are distinguishable by the 1-WL algorithm (i.e., distinct colorings). This implies that $\WLS^{\dr{k}} (\gS_1) \equiv \WLS^{\dr{k}} (\gS_2)$, but $\WLS^{\db{k+1}} (\gS_1) \not\equiv \WLS^{\db{k+1}} (\gS_2)$ (e.g., the top part in Figure~\ref{fig:example}). This contradicts our assumption that \(\WLS^{\db{k+1}} (\gS_1) \not\equiv \WLS^{\db{k+1}} (\gS_2)\) implies \(\WLS^{\dr{k}} (\gS_1) \not\equiv \WLS^{\dr{k}} (\gS_2)\).
\paragraph{Proof of Equation~\ref{eq:wls_main_2}} \vspace{-0.3cm}
For the sake of contradiction, assume that whenever \(\WLS^{\db{k+1}} (\gS_1) \equiv \WLS^{\db{k+1}} (\gS_2)\), it must follow that \(\WLS^{\dr{k}} (\gS_1) \equiv \WLS^{\dr{k}} (\gS_2)\). Let \(\gG\) be a global graph, and consider its two subgraphs \(\gS_1\) and \(\gS_2\). Suppose that in \(\gG\), $(k+1)$-hop neighborhoods of \(\gS_1\) and \(\gS_2\) are identical, $\WLS^{\db{k+1}} (\gS_1) \equiv \WLS^{\db{k+1}} (\gS_2)$. However, within the $k$-hop neighborhoods, the local structures can differ such that the rooted subtree patterns of \(\gS_1\) and \(\gS_2\) up to height \(T > k\) are not identical, $\WLS^{\dr{k}} (\gS_1) \not\equiv \WLS^{\dr{k}} (\gS_2)$ (e.g., the bottom part in Figure~\ref{fig:example}). This contradicts our assumption that \(\WLS^{\db{k+1}} (\gS_1) \equiv \WLS^{\db{k+1}} (\gS_2)\) implies \(\WLS^{\dr{k}} (\gS_1) \equiv \WLS^{\dr{k}} (\gS_2)\).
\end{proof}



\begin{figure}[t]
  \centering
  \includegraphics[width=0.89\textwidth]{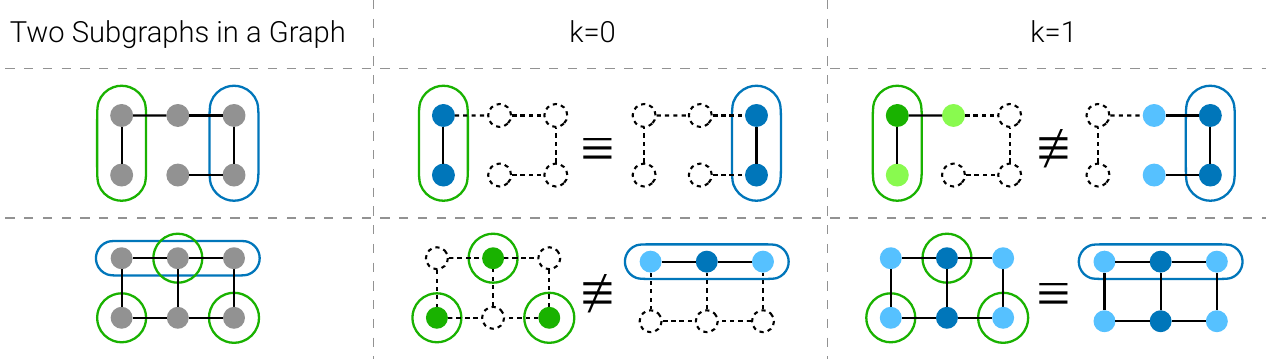}
  \vspace{-0.2cm}
    \caption{Example pairs of subgraphs that $\WLS^{k}$ produces equivalent colorings while $\WLS^{k+1}$ does not, and vice versa (where $k=0$). The gray area represents the subgraph.}
    \label{fig:example}
\end{figure}


Proposition~\ref{prop:wls_main} demonstrates that $\WLS^{k}$ cannot represent all the structural information of a smaller $k'$-hop structure ($k' < k$) from the perspective of graph isomorphism. This limitation suggests that relying solely on a single $k$ for $\WLS^k$ may be insufficient for encoding comprehensive information from various levels of structures. To address this, we propose combining $\WLS^{k}$ for multiple values of $k$, allowing the representation to capture both local and global structures effectively. Building on this insight, we design WLKS-$\sK$, a mixture of WLKS for multiple hops $k \in \sK$ where its kernel matrix $\mK_{\WLS}\text{-}\sK$ is a linear combination of $\mK_{\WLS}^{k}$.
\begin{equation}
\mK_{\WLS}\text{-}\sK \textstyle = \sum_{k \in \sK} \alpha_{k} \mK_{\WLS}^{k}\quad \where\ \alpha_{k} \in \sR^{+}.
\end{equation}
Note that WLKS-$\sK$ can be defined even when only one $k$ is used (e.g., $\mK_{\WLS}\text{-}\{0\} = \alpha_0 \mK_{\WLS}^{0}$ for WLKS-$\{0\}$). WLKS-$\sK$ is a valid kernel since a positive linear combination of p.s.d. kernels is p.s.d.~\citep{shervashidze2011weisfeiler}.

\subsection{Selecting $k$ for Minimal Complexity}

In WLKS, selecting appropriate values of $k$ during the $k$-hop subgraph sampling is crucial for balancing expressive power and complexity. As the number of nodes in the $k$-hop neighborhood grows exponentially with increasing $k$~\citep{hamilton2017inductive}, an unbounded increase in $k$ can result in substantial computation and memory overhead. To mitigate this, we strategically limit the choice of $k$ to two specific values: $k = 0$ and $k = D$, where $D$ is the diameter of the global graph $\gG$.

When $k = 0$, the WLKS consumes the least computation and memory by using only the internal structure of the subgraph without neighborhood sampling. In contrast, when $k$ is set to diameter $D$, every subgraph has the same $k$-hop neighborhood, which is the global graph $\gG$; thus, the WLS is performed just once on $\gG$ without per-subgraph computations. By using $0$ and $D$, WLKS-$\{0,D\}$ can capture both the local and the largest global structure of subgraphs. This approach offers a practical model that balances expressive power and efficiency, avoiding excessive computation and memory consumption from intermediate $k$ values.

\subsection{Computational Complexity}

The original WL Kernel has a computational complexity of $\gO \left( T \sum_i E_i^{\sub} + M T \sum_i N_i^{\sub} \right)$ for $M$ subgraphs, $T$ iterations, and the number of nodes $N_i^{\sub}$ and edges $E_i^{\sub}$ of the subgraph $i$~\citep{shervashidze2009fast}. When $k$ is 0, a set of subgraphs is identical to a set of individual graphs, so its complexity is the same as the original's. When $k$ is $D$, after performing the WL algorithm on the global graph once (i.e., $\gO (TE)$), the coloring of each subgraph is aggregated to a histogram (i.e., $\gO (\sum_i N_i^{\sub})$). Thus, the computational complexity of WLKS-$\{ 0, D \}$ is $\gO \left( T (E + \sum_i E_i^{\sub}) + M T \sum_i N_i^{\sub} \right)$.

We note that WLKS-$\{ 0, D \}$ do not perform $k$-hop neighborhood sampling, which adds a complexity of \( \gO (N^{\sub,k} + E^{\sub,k}) \) per subgraph from a breadth-first search from $\sV^{\sub}$. Learning SVM with pre-computed kernels has a complexity of $\gO (M^2)$ dependent on the number of subgraphs $M$, but this step is typically secondary to the WLS in practice.

\section{Incorporating Continuous Features for WLKS}
WLKS is designed to capture structural information but can be simply integrated with continuous features. This section introduces four methods to incorporate continuous features for WLKS.

\subsection{Combination with Kernels on Continuous Features}

WLKS can be linearly combined with kernel matrices $\mK_{\mX}$ derived from continuous features as in Equation~\ref{eq:wlks_cont}. This combination enables the model to account for structure and feature similarities between subgraphs. One straightforward way is to directly compute a kernel on features, which measures the similarity between subgraphs based on their feature vectors. Another approach involves applying the Continuous Weisfeiler-Lehman operator~\citep{togninalli2019wasserstein} to features, generating a kernel matrix. This operator extends the original WL framework to continuously attributed graphs.
\begin{equation}\label{eq:wlks_cont}
\alpha_{\text{structure}} \cdot \mK_{\WLS}\text{-}\sK + \alpha_{\text{feature}} \cdot \mK_{\mX}\quad \where\ \alpha_{\cdot} \in \sR^{+}.
\end{equation}
In both cases, the kernel matrix from continuous features tends to be denser and has a different scale compared to those from the WL histogram. To address this, we standardize the features before creating kernels and use the RBF and linear kernels.

\subsection{GNNs with the WLKS Kernel Matrix as Adjacency Matrix}

Another way to integrate features with WLKS is to use the kernel matrix as an adjacency matrix. Specifically, we consider the WLKS kernel matrix $\mK_{\WLS}$ as the adjacency matrix of a weighted graph where subgraphs $\sS$ serve as nodes. The rationale for this approach is that a kernel represents the similarity between data points.

In this graph, the edge weight between subgraphs $i$ and $j$ corresponds to $\mK_{\WLS}[i, j]$. By applying deep GNNs to this graph, we can leverage the expressive power of WLKS for structural information and the capabilities of GNNs for feature representation. For this paper, we adopt state-of-the-art GNN-based models, S2N+0 and S2N+A~\citep{kim2024translating}, for the graph created by WLKS-$\{0,D\}$ as an instantiation of this approach.

Given the original feature $\mX \in \sR^{N \times \text{\# features}}$, in S2N+0, the hidden feature $\mH \in \sR^{M \times \text{\# features}}$ is a sum of original features in the subgraph, and then a GNN on $\mK_{\WLS} \in \sR^{M \times M}$ is applied to get the logit matrix $\mY \in \sR^{M \times \text{\# classes}}$ for the prediction. S2N+A first encodes each subgraph as an individual graph with a GNN, readout its output to get the hidden feature $\mH$, then the other GNN on $\mK_{\WLS}$ is applied for the prediction. Formally,
\begin{align}
&\text{WLKS for S2N+0:  }
    \quad
    \mH[i, :] \textstyle
        = \mathbf{1}_{N^{\sub}}^{\top}   \mX[\sV_i^{\sub}, :],
    \ \ 
    \mY = \gnn (\mH, \mK_{\WLS}),
    \\
&\text{WLKS for S2N+A:  }
    \quad
    \mH[i, :] \textstyle
        = \mathbf{1}_{N^{\sub}}^{\top}
        \gnn_{1} ( \mX[\sV_i^{\sub}, :], \sA_{i}^{\sub}),
    \ \ 
    \mY = \gnn_{2} (\mH, \mK_{\WLS}),
\end{align}
where $\mathbf{1}_{n} \in \sR^{n \times 1}$ is a vector of ones. Since the kernel matrix is dense for GPUs, we sparsify and normalize it using the same method in the S2N's paper.

\section{Experiments}
\renewcommand{\arraystretch}{1.2}

\begin{table}[t]
\caption{Statistics of real-world and synthetic datasets.}
\centering
\resizebox{\columnwidth}{!}{%
\begin{tabular}{lllllllll}
\hline
 & \PPIBP & \HPONeuro & \HPOMetab & \EMUser & \Density & \CutRatio & \Coreness & \Component \\ \hline
\# nodes in $\gG$ & 17,080 & 14,587 & 14,587 & 57,333 & 5,000 & 5,000 & 5,000 & 19,555 \\
\# edges in $\gG$ & 316,951 & 3,238,174 & 3,238,174 & 4,573,417 & 29,521 & 83,969 & 118,785 & 43,701 \\
\# subgraphs ($\gS$) & 1,591 & 2,400 & 4,000 & 324 & 250 & 250 & 221 & 250 \\
\# nodes / $\gS$ & 10.2$_{\pm 10.5}$ & 14.4$_{\pm 6.2}$ & 14.8$_{\pm 6.5}$ & 155.4$_{\pm 100.2}$ & 20.0$_{\pm 0.0}$ & 20.0$_{\pm 0.0}$ & 20.0$_{\pm 0.0}$ & 74.2$_{\pm 52.8}$ \\
\# components / $\gS$ & 7.0$_{\pm 5.5}$ & 1.6$_{\pm 0.7}$ & 1.5$_{\pm 0.7}$ & 52.1$_{\pm 15.3}$ & 3.8$_{\pm 3.7}$ & 1.0$_{\pm 0.0}$ & 1.0$_{\pm 0.0}$ & 4.9$_{\pm 3.5}$ \\
Density ($\gG$)   & 0.0022              & 0.0304              & 0.0304              & 0.0028              & 0.0024              & 0.0067              & 0.0095              & 0.0002              \\ 
Avg. density ($\gS$) & 0.216 & 0.757 & 0.767 & 0.010 & 0.232 & 0.945 & 0.219 & 0.150 \\ 
\# classes & 6 & 10 & 6 & 2 & 3 & 3 & 3 & 2 \\
Labels & Single & Multi & Single & Single & Single & Single & Single & Single \\
Dataset splits & 80/10/10 & 80/10/10 & 80/10/10 & 70/15/15 & 80/10/10 & 80/10/10 & 80/10/10 & 80/10/10 \\ \hline
\end{tabular}%
}
\label{tab:dataset_stats}
\end{table}

This section outlines the experimental setup, covering the datasets, training details, and baselines.

\paragraph{Datasets}

We employ four real-world datasets (\PPIBP, \HPONeuro, \HPOMetab, and \EMUser) and four synthetic datasets (\Density, \CutRatio, \Coreness, and \Component) introduced by \citet{alsentzer2020subgraph}. Given the global graph $\gG$ and subgraphs $\sS$, the goal of the real-world benchmark is subgraph classification on various domains: protein-protein interactions (\PPIBP), medical knowledge graphs (\HPONeurob and \HPOMetab), and social networks (\EMUser). For synthetic benchmarks, the goal is to determine the structural properties (density, cut ratio, the average core number, and the number of components) formulated as a classification. Note that WLKS does not need pretrained embeddings. We summarize dataset statistics in Table~\ref{tab:dataset_stats}.

\paragraph{Models}

We experiment with five WLKS-$\sK$ where $\sK$ is $\{0\}, \{1\}, \{2\}, \{D\}, \{0, D\}$. Coefficients $\alpha$ is set to 1 when one $k$ is selected, and $\alpha_0 + \alpha_D = 1$ for WLKS-$\{0, D\}$. We do a grid search of five hyperparameters: the number of iterations ($\{1, 2, 3, 4, 5\}$), whether to combine kernels of all iterations, whether to normalize histograms, L2 regularization ($\{ 2^3 / 100, 2^4 / 100, ..., 2^{14} / 100 \}$), and the coefficient $\alpha_0 (\{ 0.999, 0.99, 0.9, 0.5, 0.1, 0.01, 0.001 \})$. When combining with kernels on continuous features (Equation~\ref{eq:wlks_cont}), we tune $\alpha_{\text{feature}}$ from the space of $\{ 0.0001, 0.001, 0.01, 0.05, 0.1, 0.15, 0.2, 0.25 \}$ and set $\alpha_{\text{structure}} = 1 / (1 + \alpha_{\text{feature}})$. For fusing WLKS-$\{0,D\}$ to S2N, we follow the GCNII-based~\citep{chen2020simple} architecture and settings presented in \citet{kim2024translating}.

\paragraph{Baselines}

We use state-of-the-art GNN-based models for subgraph classification tasks as baselines: Subgraph Neural Network~\citep[SubGNN;][]{alsentzer2020subgraph}, GNN with LAbeling trickS for Subgraph~\citep[GLASS;][]{wang2022glass}, Variational Subgraph Autoencoder~\citep[VSubGAE;][]{liu2023position}, Stochastic Subgraph Neighborhood Pooling~\citep[SSNP;][]{jacob2023stochastic}, and Subgraph-To-Node Translation~\citep[S2N;][]{kim2024translating}. Baseline results are taken from the corresponding research papers.

\paragraph{Efficiency Measurement}

When measuring the complete training time, we run models of the best hyperparameters from each model's original code, including batch sizes and total epochs, using Intel(R) Xeon(R) CPU E5-2640 v4 and a single GeForce GTX 1080 Ti (for deep GNNs).

\paragraph{Implementation}

All models are implemented with PyTorch~\citep{paszke2019pytorch} and PyTorch Geometric~\citep{Fey2019Fast}. We use the implementation of Support Vector Machines (SVMs) in Scikit-learn~\citep{pedregosa2011scikit}. 

\section{Results and Discussions}
In this section, we compare the classification performance and efficiency of WLKS and baselines. In addition, the performance of WLKS according to $\sK$ is demonstrated to exhibit the usefulness of the kernel combination. Finally, we investigate how integrating structures and features across subgraph datasets affects downstream performance.

\renewcommand{\arraystretch}{1.2}

\begin{table}[t]
\centering
\caption{Mean performance in micro F1-score on real-world and synthetic datasets over 10 runs. A subscript indicates the standard deviation. The higher the performance, the darker the blue color. The results of baselines are reprinted from respective papers.}
\label{tab:wlks_results}
\resizebox{\columnwidth}{!}{%
\begin{tabular}{ccccccccc}
\hline
Model &
  \PPIBP &
  \HPONeuro &
  \HPOMetab &
  \EMUser &
  \Density &
  \CutRatio &
  \Coreness &
  \Component \\ \hline
SubGNN &
  \cellcolor[HTML]{FFFFFF}$59.9_{\pm 2.4}$ &
  \cellcolor[HTML]{FFFFFF}$63.2_{\pm 1.0}$ &
  \cellcolor[HTML]{FFFFFF}$53.7_{\pm 2.3}$ &
  \cellcolor[HTML]{FFFFFF}$81.4_{\pm 4.6}$ &
  \cellcolor[HTML]{CCDCF9}$91.9_{\pm 1.6}$ &
  \cellcolor[HTML]{C2D5F7}$62.9_{\pm 3.9}$ &
  \cellcolor[HTML]{EEF3FD}$65.9_{\pm 9.2}$ &
  \cellcolor[HTML]{FFFFFF}$95.8_{\pm 9.8}$ \\
GLASS &
  \cellcolor[HTML]{E2EBFC}$61.9_{\pm 0.7}$ &
  \cellcolor[HTML]{6D9EEB}$68.5_{\pm 0.5}$ &
  \cellcolor[HTML]{92B7F1}$61.4_{\pm 0.5}$ &
  \cellcolor[HTML]{C9DAF8}$88.8_{\pm 0.6}$ &
  \cellcolor[HTML]{C9DAF8}$93.0_{\pm 0.9}$ &
  \cellcolor[HTML]{6D9EEB}$93.5_{\pm 0.6}$ &
  \cellcolor[HTML]{C9DAF8}$84.0_{\pm 0.9}$ &
  \cellcolor[HTML]{6D9EEB}$100.0_{\pm 0.0}$ \\
VSubGAE &
  - &
  \cellcolor[HTML]{DEE8FB}$65.2_{\pm 1.4}$ &
  \cellcolor[HTML]{E3ECFC}$56.3_{\pm 0.9}$ &
  \cellcolor[HTML]{E5EEFC}$85.0_{\pm 3.5}$ &
  - &
  - &
  - &
  - \\
SSNP-NN &
  \cellcolor[HTML]{C6D8F8}$63.6_{\pm 0.7}$ &
  \cellcolor[HTML]{7BA7ED}$68.2_{\pm 0.4}$ &
  \cellcolor[HTML]{C9DAF8}$58.7_{\pm 1.0}$ &
  \cellcolor[HTML]{C9DAF8}$88.8_{\pm 0.5}$ &
  - &
  - &
  - &
  - \\
S2N+0$_{\text{ GCNII}}$ &
  \cellcolor[HTML]{CADBF9}$63.5_{\pm 2.4}$ &
  \cellcolor[HTML]{C9DAF8}$66.4_{\pm 1.1}$ &
  \cellcolor[HTML]{8EB4F0}$61.6_{\pm 1.7}$ &
  \cellcolor[HTML]{DAE6FB}$86.5_{\pm 3.2}$ &
  \cellcolor[HTML]{FFFFFF}$67.2_{\pm 2.4}$ &
  \cellcolor[HTML]{FFFFFF}$56.0_{\pm 0.0}$ &
  \cellcolor[HTML]{FFFFFF}$57.0_{\pm 4.9}$ &
  \cellcolor[HTML]{6D9EEB}$100.0_{\pm 0.0}$ \\
S2N+A$_{\text{ GCNII}}$ &
  \cellcolor[HTML]{BED3F7}$63.7_{\pm 2.3}$ &
  \cellcolor[HTML]{72A1EC}$68.4_{\pm 1.0}$ &
  \cellcolor[HTML]{6D9EEB}$63.2_{\pm 2.7}$ &
  \cellcolor[HTML]{C3D6F8}$89.0_{\pm 1.6}$ &
  \cellcolor[HTML]{C3D6F8}$93.2_{\pm 2.6}$ &
  \cellcolor[HTML]{FFFFFF}$56.0_{\pm 0.0}$ &
  \cellcolor[HTML]{B4CDF5}$85.7_{\pm 5.8}$ &
  \cellcolor[HTML]{6D9EEB}$100.0_{\pm 0.0}$ \\ \hline
WLKS-$\{0,D\}$ &
  \cellcolor[HTML]{6D9EEB}$64.8_{\pm 0.0}$ &
  \cellcolor[HTML]{DCE7FB}$65.3_{\pm 0.0}$ &
  \cellcolor[HTML]{D2E0FA}$57.9_{\pm 0.0}$ &
  \cellcolor[HTML]{6D9EEB}$91.8_{\pm 0.0}$ &
  \cellcolor[HTML]{6D9EEB}$96.0_{\pm 0.0}$ &
  \cellcolor[HTML]{C9DAF8}$60.0_{\pm 0.0}$ &
  \cellcolor[HTML]{6D9EEB}$91.3_{\pm 0.0}$ &
  \cellcolor[HTML]{6D9EEB}$100.0_{\pm 0.0}$ \\ \hline
\end{tabular}%
}
\end{table}

\begin{table}[t]
\centering
\caption{Runtime in seconds of our model and baselines for the entire training stage and 1-epoch inference (validation set) for real-world datasets.}
\label{tab:wl4s_tt}
\resizebox{\textwidth}{!}{%
\begin{tabular}{ccccc|cccc}
\hline
Stage                   & \multicolumn{4}{c|}{Entire Training}     & \multicolumn{4}{c}{ Inference (1 epoch)}  \\ \hline
Model                   & \PPIBP & \HPONeuro & \HPOMetab & \EMUser & \PPIBP & \HPONeuro & \HPOMetab & \EMUser \\ \hline
SubGNN                  & N/A    & 1798.2    & 1082.1    & 108.1   & N/A    & 432.9     & 257.1     & 35.8    \\
GLASS                   & 1009.6 & 2462.6    & 1397.0    & 4597.4  & 8.2    & 27.0      & 26.4      & 39.0    \\
S2N+0$_{\text{ GCNII}}$ & 16.7   & 36.7      & 37.1      & 31.0    & 9.9    & 9.3       & 8.3       & 14.6    \\
S2N+A$_{\text{ GCNII}}$ & 14.9   & 78.0      & 72.2      & 39.0    & 8.4    & 11.1      & 9.6       & 13.4    \\ \hline
WLKS-\{0,D\}            & 3.5    & 25.2      & 10.9      & 9.6     & 1.0    & 11.7      & 2.7       & 1.8     \\ \hline
\end{tabular}%
}
\end{table}

\paragraph{Performance and Efficiency}

In Table~\ref{tab:wlks_results}, the classification performance of \WLKSZD{} and baselines on eight datasets is summarized. Our results show that our model outperforms the best-performing baseline in five out of eight datasets. Specifically, \WLKSZD{} achieves the highest micro F1-score on \PPIBP, \EMUser, \Density, \Coreness, and \Component. For \HPONeuro, \HPOMetab, and \CutRatio, our model shows similar performance to SubGNN but relatively lower performance than the state-of-the-art model.

In terms of efficiency, we present the training and inference time of our model and four representative baselines on real-world datasets in Table~\ref{tab:wl4s_tt}. Specifically, we measure the runtime for the entire training stage (including data and model loading, training steps, validation steps, and logging) and 1-epoch inference on the validation set. Note that an experiment on \PPIBPb with SubGNN cannot be conducted since it takes more than 48 hours in pre-computation. \WLKSZD{} demonstrates significantly faster training and inference times across all real-world datasets compared to other models (e.g., the shorter training time of $\times 0.01$ -- $\times 0.25$ and inference time of $\times 0.12$ -- $\times 0.43$). This metric does not include the pre-computation or embedding pretraining required in baselines, so the actual training of WLKS is more efficient. Additionally, WLKS does not require a GPU in training and inference, unlike other GNN baselines. 

Our results highlight the balance between efficiency and representation quality of WL histograms. Despite the loss of structural information from algorithmic simplicity, the empirical evidence underscores the task-relevant nature of the retained information in \WLKSZD{}. This is reflected in the superior performance of \WLKSZD{} with significantly smaller computational requirements on five out of eight datasets. However, the WL histogram lacks expressiveness in some subgraph tasks. In these cases, the expressiveness can be enhanced by using the structural encoding (or labeling tricks)~\citep{zhang2018link,li2020distance,zhang2021labeling,dwivedi2022graph}. For example, on the \CutRatiob dataset, where the performance of WLKS is low compared to the state-of-the-art, a linear combination with the inner product kernel of Random Walk Structural Encoding~\citep{dwivedi2022graph} significantly improves the performance from 60.0 to 96.0. However, this improvement is not observed on other datasets. Detailed discussions are in Appendix~\ref{appendix:se}.

\begin{table}[t]
\centering
\caption{Mean performance of WLKS-$\sK$ in micro F1-score by $\sK$: $\{0\}$, $\{1\}$, $\{2\}$, $\{D\}$, and $\{0,D\}$. The standard deviations are omitted (all 0). The higher the performance, the darker the blue color.}
\label{tab:wlks_k}
\resizebox{\columnwidth}{!}{%
\begin{tabular}{ccccccccc}
\hline
Model &
  \PPIBP &
  \HPONeuro &
  \HPOMetab &
  \EMUser &
  \Density &
  \CutRatio &
  \Coreness &
  \Component \\ \hline
WLKS-$\{0,D\}$ &
  \cellcolor[HTML]{6D9EEB}$64.8$ &
  \cellcolor[HTML]{6D9EEB}$65.3$ &
  \cellcolor[HTML]{6D9EEB}$57.9$ &
  \cellcolor[HTML]{6D9EEB}$91.8$ &
  \cellcolor[HTML]{6D9EEB}$96.0$ &
  \cellcolor[HTML]{6D9EEB}$60.0$ &
  \cellcolor[HTML]{6D9EEB}$91.3$ &
  \cellcolor[HTML]{6D9EEB}$100.0$ \\ \hline
WLKS-$\{0\}$ &
  \cellcolor[HTML]{FFFFFF}$34.0$ &
  \cellcolor[HTML]{FFFFFF}$31.4$ &
  \cellcolor[HTML]{FFFFFF}$26.4$ &
  \cellcolor[HTML]{FFFFFF}$67.3$ &
  \cellcolor[HTML]{6D9EEB}$96.0$ &
  \cellcolor[HTML]{FFFFFF}$36.0$ &
  \cellcolor[HTML]{75A3ED}$87.0$ &
  \cellcolor[HTML]{6D9EEB}$100.0$ \\
WLKS-$\{1\}$ &
  \cellcolor[HTML]{F7F9FE}$39.0$ &
  OOM &
  OOM &
  \cellcolor[HTML]{E2EBFC}$79.6$ &
  \cellcolor[HTML]{FFFFFF}$68.0$ &
  \cellcolor[HTML]{C9DAF8}$56.0$ &
  \cellcolor[HTML]{FFFFFF}$39.1$ &
  \cellcolor[HTML]{6D9EEB}$100.0$ \\
WLKS-$\{2\}$ &
  \cellcolor[HTML]{C9DAF8}$64.2$ &
  OOM &
  OOM &
  \cellcolor[HTML]{C9DAF8}$89.8$ &
  \cellcolor[HTML]{FFFFFF}$68.0$ &
  \cellcolor[HTML]{C9DAF8}$56.0$ &
  \cellcolor[HTML]{FFFFFF}$39.1$ &
  \cellcolor[HTML]{6D9EEB}$100.0$ \\
WLKS-$\{D\}$ &
  \cellcolor[HTML]{C9DAF8}$64.2$ &
  \cellcolor[HTML]{C9DAF8}$65.1$ &
  \cellcolor[HTML]{6D9EEB}$57.9$ &
  \cellcolor[HTML]{C9DAF8}$89.8$ &
  \cellcolor[HTML]{FFFFFF}$68.0$ &
  \cellcolor[HTML]{C9DAF8}$56.0$ &
  \cellcolor[HTML]{FFFFFF}$39.1$ &
  \cellcolor[HTML]{6D9EEB}$100.0$ \\ \hline
\end{tabular}%
}
\end{table}

\paragraph{Performance of WLKS-$\sK$ by $\sK$}

We highlight the importance of selecting the appropriate $\sK$ in Table~\ref{tab:wlks_k}. Specifically, the performance of WLKS-$\sK$ varies significantly depending on the choice of $\sK$. WLKS-$\{0,D\}$, which combines kernels of $0$ and $D$, consistently delivers strong results across datasets. WLKS-$\{0\}$ and WLKS-$\{D\}$ perform well independently in certain datasets, but their combination makes the better performance. This is clearly demonstrated in the \Corenessb experiment for predicting the average core number of nodes within a subgraph. This benchmark requires modeling internal structures, global structures, and subgraph positions. A significant improvement by combining kernels on \Corenessb aligns with our research motivation of capturing arbitrary interactions between and within subgraph structures. This ability is necessary when multiple $k$-hop neighborhoods are associated with the labels of the subgraph, and the performance can be improved from the complementary nature of WLKS capturing different $k$-hop structures.

The selection of $k$ is analogous to determining the number of layers in GNNs and can be treated as a hyperparameter optimized for specific tasks. Intermediate values of $k$ may capture important substructures in datasets with large diameters. However, the model empirically performs well even when $k$ is substantially smaller than the graph's diameter $D$. For instance, \PPIBP's largest component has a diameter of 8, yet $k=2$ performs as well as $k=D$. In addition, our empirical results showed no clear relation between the size of the graph, its global density, or the average density of subgraphs (as presented in Table~\ref{tab:dataset_stats}) and task performance with different $k$ values. This suggests that such structural properties are not key factors in determining the optimal $k$. Instead, the nature of the task and its structural requirements should guide the selection of $k$.

\begin{table}[t]
\centering
\caption{Mean performance in micro F1-score of WLKS variants integrated with continuous features over 10 runs. Used S2N models are based on GCNII~\citep{chen2020simple}. The higher the performance, the darker the blue color.}
\label{tab:wl4s_s2n}
\resizebox{\textwidth}{!}{%
\begin{tabular}{ccccccccc}
\hline
Model &
  \multicolumn{1}{l}{\PPIBP} &
  \multicolumn{1}{l}{\HPONeuro} &
  \multicolumn{1}{l}{\HPOMetab} &
  \multicolumn{1}{l}{\EMUser} &
  \multicolumn{1}{l}{\Density} &
  \multicolumn{1}{l}{\CutRatio} &
  \multicolumn{1}{l}{\Coreness} &
  \multicolumn{1}{l}{\Component} \\ \hline
WLKS-$\{0,D\}$ &
  \cellcolor[HTML]{FFFFFF}$64.8_{\pm 0.0}$ &
  \cellcolor[HTML]{FFFFFF}$65.3_{\pm 0.0}$ &
  \cellcolor[HTML]{FFFFFF}$57.9_{\pm 0.0}$ &
  \cellcolor[HTML]{C9DAF8}$91.8_{\pm 0.0}$ &
  \cellcolor[HTML]{6D9EEB}$96.0_{\pm 0.0}$ &
  \cellcolor[HTML]{C9DAF8}$60.0_{\pm 0.0}$ &
  \cellcolor[HTML]{6D9EEB}$91.3_{\pm 0.0}$ &
  \cellcolor[HTML]{6D9EEB}$100.0_{\pm 0.0}$ \\ \hline
WLKS-$\{0,D\}$ + Kernel on $\mX$ &
  \cellcolor[HTML]{FFFFFF}$64.8_{\pm 0.0}$ &
  \cellcolor[HTML]{DBE7FB}$65.9_{\pm 0.0}$ &
  \cellcolor[HTML]{D9E5FB}$59.1_{\pm 0.0}$ &
  \cellcolor[HTML]{C9DAF8}$91.8_{\pm 0.0}$ &
  \cellcolor[HTML]{6D9EEB}$96.0_{\pm 0.0}$ &
  \cellcolor[HTML]{6D9EEB}$64.0_{\pm 0.0}$ &
  \cellcolor[HTML]{6D9EEB}$91.3_{\pm 0.0}$ &
  \cellcolor[HTML]{6D9EEB}$100.0_{\pm 0.0}$ \\
WLKS-$\{0,D\}$ + Cont. WL Kernel on $\mX$ &
  \cellcolor[HTML]{FFFFFF}$64.8_{\pm 0.0}$ &
  \cellcolor[HTML]{C9DAF8}$66.2_{\pm 0.0}$ &
  \cellcolor[HTML]{C9DAF8}$59.6_{\pm 0.0}$ &
  \cellcolor[HTML]{6D9EEB}$93.9_{\pm 0.0}$ &
  \cellcolor[HTML]{6D9EEB}$96.0_{\pm 0.0}$ &
  \cellcolor[HTML]{6D9EEB}$64.0_{\pm 0.0}$ &
  \cellcolor[HTML]{6D9EEB}$91.3_{\pm 0.0}$ &
  \cellcolor[HTML]{6D9EEB}$100.0_{\pm 0.0}$ \\
WLKS-$\{0,D\}$ for S2N+0 &
  \cellcolor[HTML]{FFFFFF}$64.8_{\pm 1.5}$ &
  \cellcolor[HTML]{C5D8F8}$66.3_{\pm 0.6}$ &
  \cellcolor[HTML]{7BA8ED}$62.4_{\pm 1.1}$ &
  \cellcolor[HTML]{FFFFFF}$86.5_{\pm 2.4}$ &
  \cellcolor[HTML]{FFFFFF}$92.0_{\pm 0.0}$ &
  \cellcolor[HTML]{F1F6FE}$51.2_{\pm 3.9}$ &
  \cellcolor[HTML]{FFFFFF}$69.6_{\pm 1.9}$ &
  \cellcolor[HTML]{6D9EEB}$100.0_{\pm 0.0}$ \\
WLKS-$\{0,D\}$ for S2N+A &
  \cellcolor[HTML]{6D9EEB}$65.4_{\pm 2.4}$ &
  \cellcolor[HTML]{6D9EEB}$68.4_{\pm 1.1}$ &
  \cellcolor[HTML]{6D9EEB}$62.9_{\pm 1.9}$ &
  \cellcolor[HTML]{DCE7FB}$90.0_{\pm 3.3}$ &
  \cellcolor[HTML]{CFDEF9}$95.6_{\pm 2.8}$ &
  \cellcolor[HTML]{FFFFFF}$48.0_{\pm 0.0}$ &
  \cellcolor[HTML]{D3E1FA}$87.4_{\pm 4.1}$ &
  \cellcolor[HTML]{6D9EEB}$100.0_{\pm 0.0}$ \\ \hline
\end{tabular}%
}
\end{table}

\paragraph{Performance of WLKS Variants Integrated with Continuous Features}

Table~\ref{tab:wl4s_s2n} presents the results of WLKS variants combining WLKS-$\{0,D\}$ and continuous features. The performance of WLKS-$\{0,D\}$ for S2N improves over vanilla WLKS-$\{0,D\}$ on \PPIBP, \HPONeuro, and \HPOMetab. However, performance decreases on \EMUser, \Density, \CutRatio, and \Coreness. Applying GNNs to the WLKS kernel matrix requires kernel sparsification, which leads to additional loss of structural information. The enhanced features provided by deep neural networks can mitigate this trade-off. We interpret that the former set of benchmarks prioritizes features over structure, while the latter relies more on structural information. Using kernels on continuous features improves performance on \HPONeuro, \HPOMetab, \EMUser, and \CutRatio. Notably, on \EMUser, it achieves the best performance of 93.9 among all methods. For other datasets, no changes in performance are observed. Unlike the combination of GNN-based models, which can leverage the feature processing of neural networks, kernels on features provide limited performance gains from features.



\begin{figure*}[t]
  \centering
  \begin{subfigure}[t]{0.24\textwidth}
    \centering
    \includegraphics[width=\textwidth]{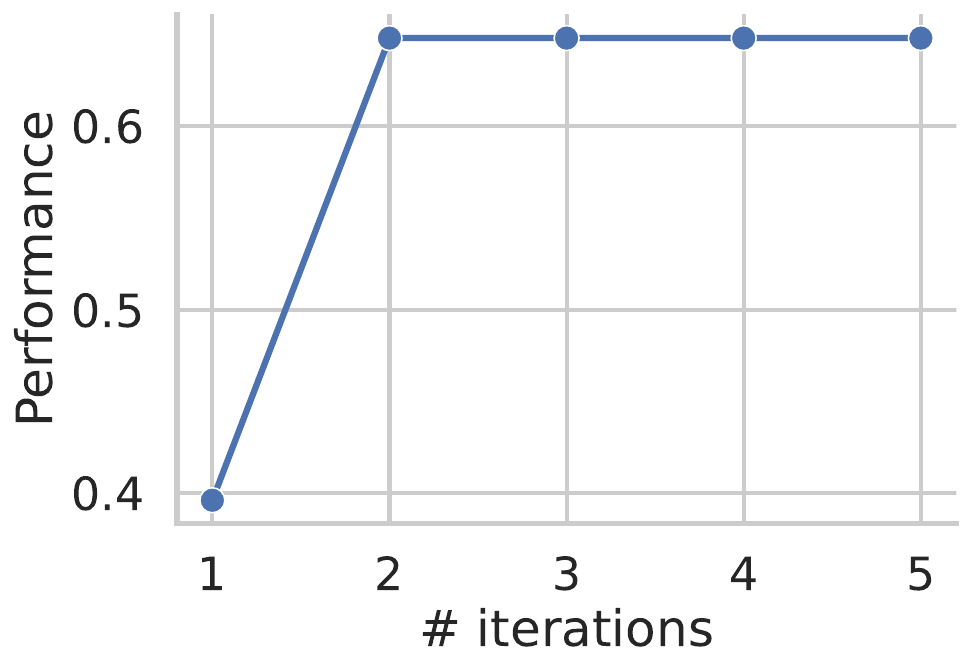}
    \vspace{-0.55cm}
    \caption{\PPIBP}
    \vspace{0.0cm}
  \end{subfigure}
  \begin{subfigure}[t]{0.24\textwidth}
    \centering
    \includegraphics[width=\textwidth]{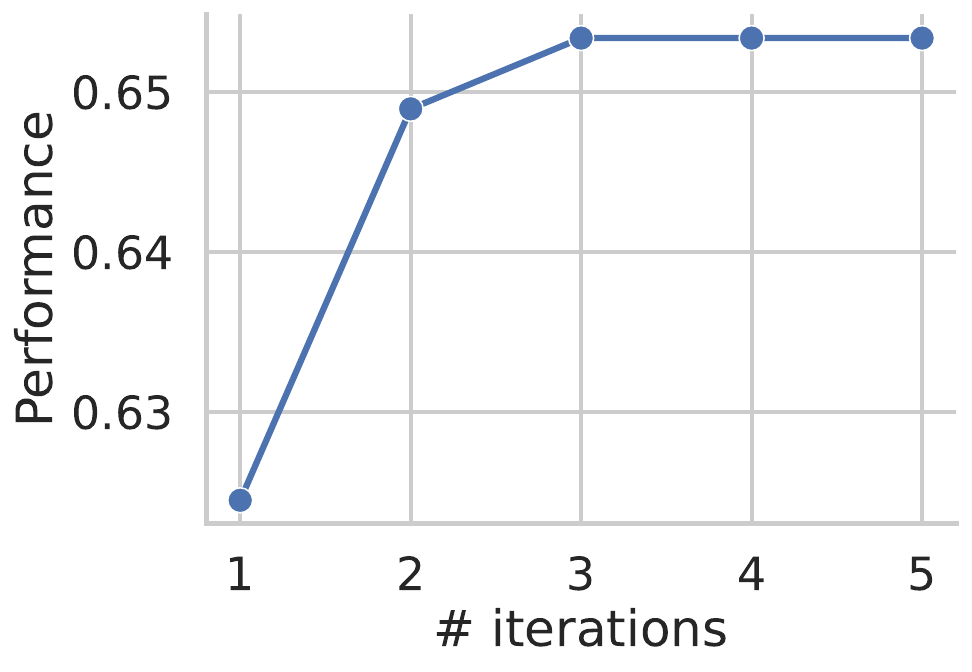}
    \vspace{-0.55cm}
    \caption{\HPONeuro}
    \vspace{0.0cm}
  \end{subfigure}
  \begin{subfigure}[t]{0.24\textwidth}
    \centering
    \includegraphics[width=\textwidth]{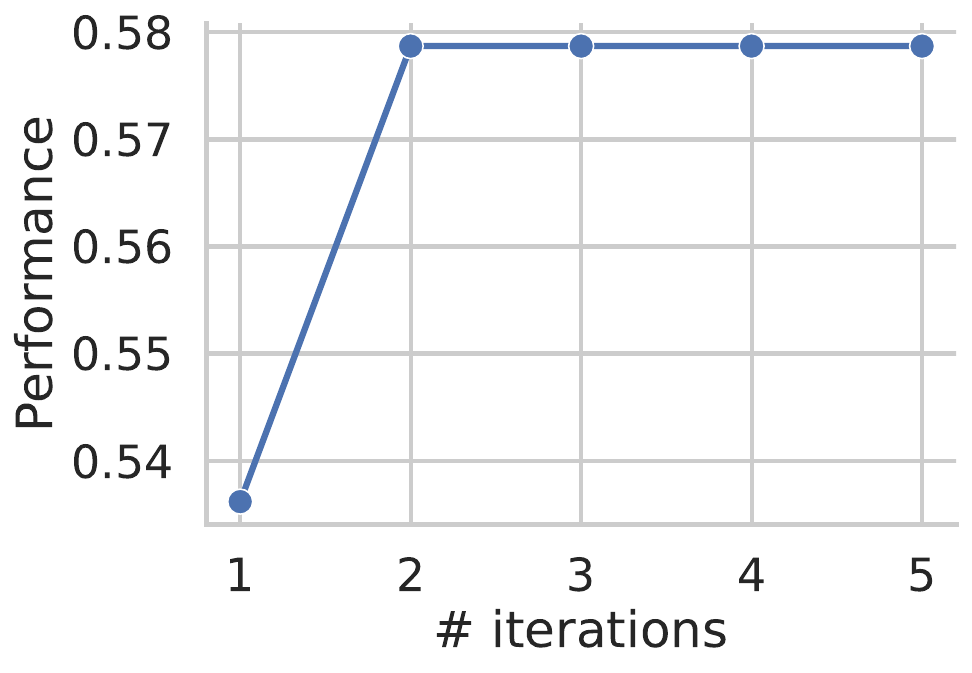}
    \vspace{-0.55cm}
    \caption{\HPOMetab}
    \vspace{0.0cm}
  \end{subfigure}
  \begin{subfigure}[t]{0.24\textwidth}
    \centering
    \includegraphics[width=\textwidth]{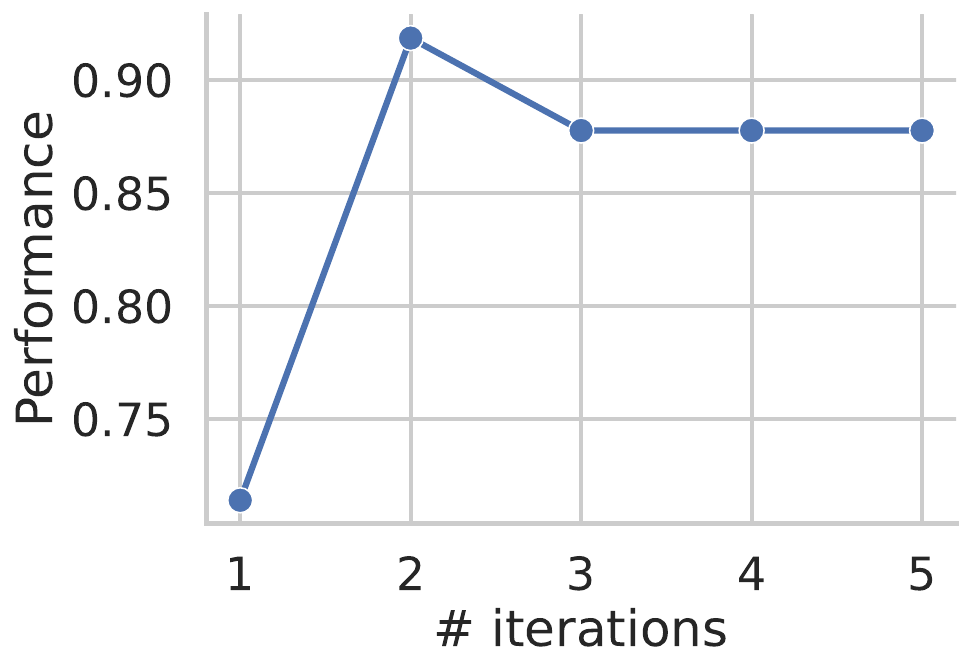}
    \vspace{-0.55cm}
    \caption{\EMUser}
    \vspace{0.0cm}
  \end{subfigure}
  \vspace{-0.09cm}
  \caption{Performance of WLKS-$\{0,D\}$ by the number of iterations $T$.}
  \label{fig:sensitivity_num_iters}
\end{figure*}

\begin{figure*}[t]
  \centering
  \begin{subfigure}[t]{0.24\textwidth}
    \centering
    \includegraphics[width=\textwidth]{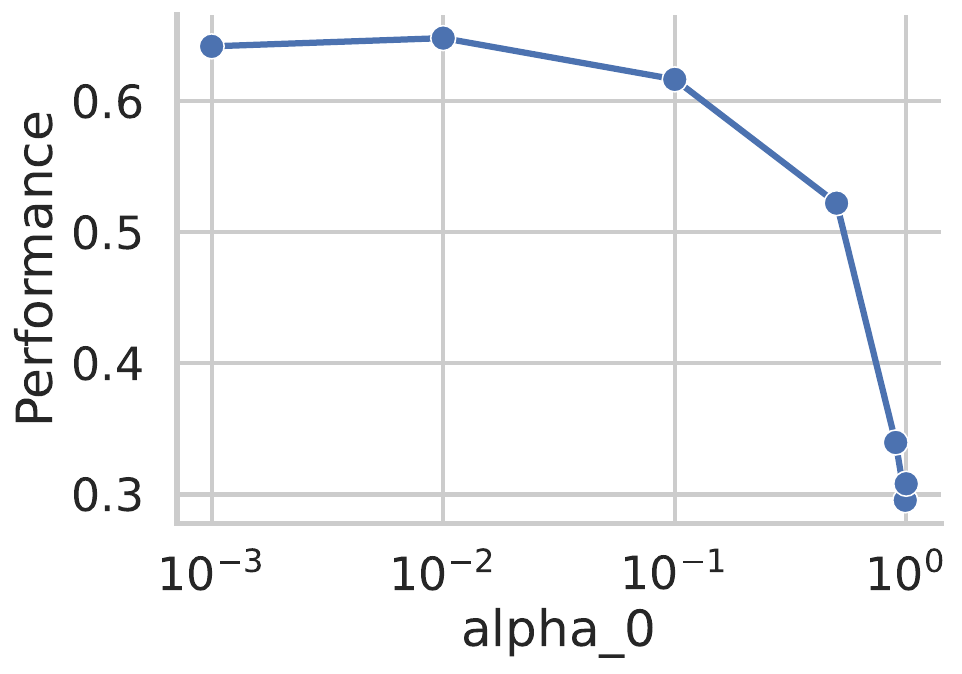}
    \vspace{-0.55cm}
    \caption{\PPIBP}
    \vspace{0.0cm}
  \end{subfigure}
  \begin{subfigure}[t]{0.24\textwidth}
    \centering
    \includegraphics[width=\textwidth]{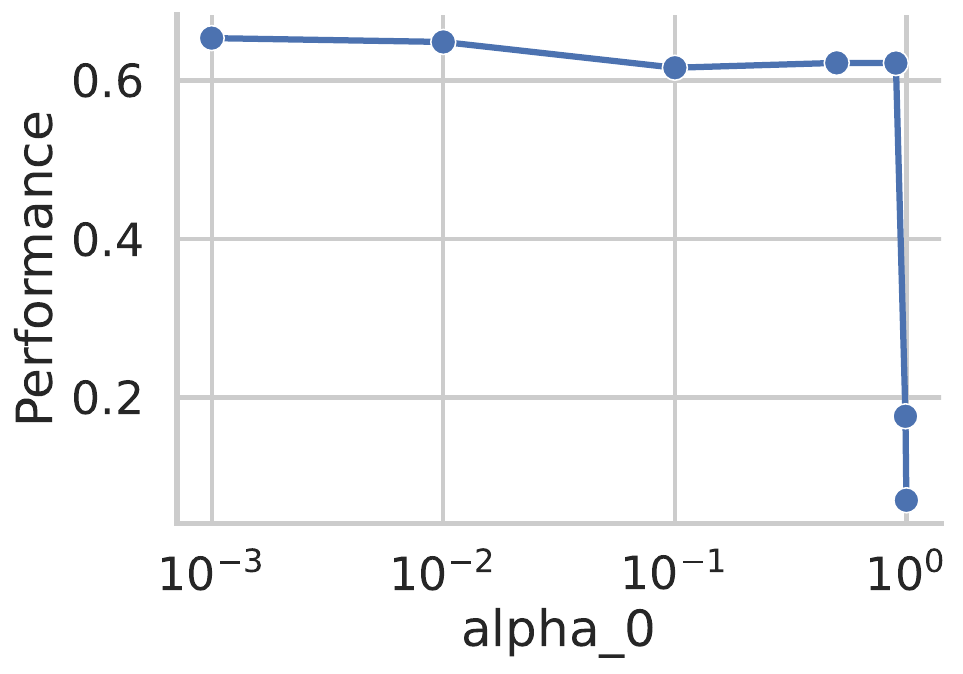}
    \vspace{-0.55cm}
    \caption{\HPONeuro}
    \vspace{0.0cm}
  \end{subfigure}
  \begin{subfigure}[t]{0.24\textwidth}
    \centering
    \includegraphics[width=\textwidth]{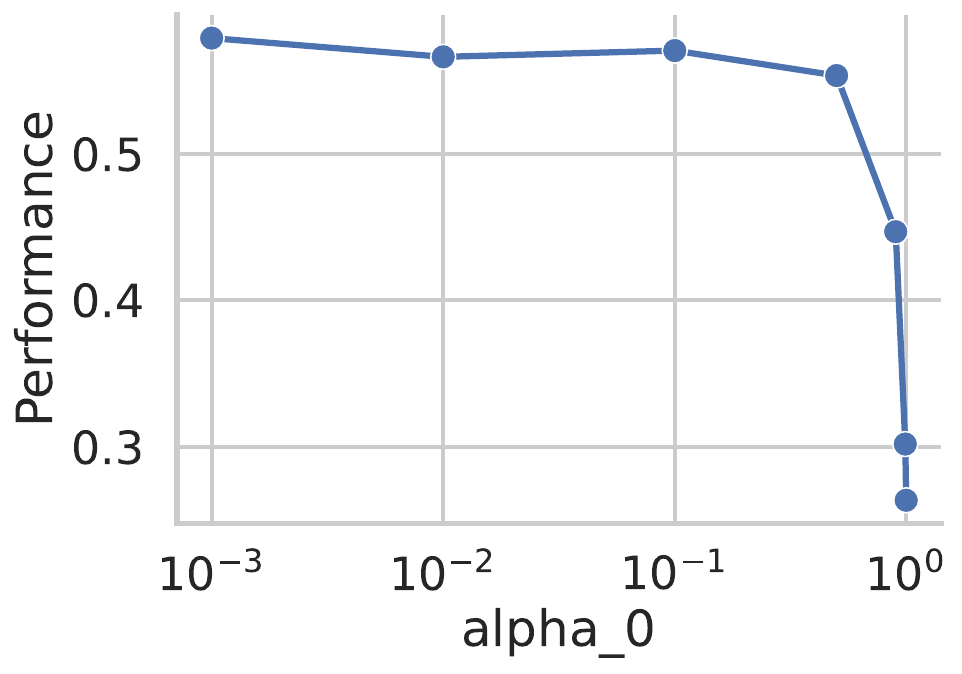}
    \vspace{-0.55cm}
    \caption{\HPOMetab}
    \vspace{0.0cm}
  \end{subfigure}
  \begin{subfigure}[t]{0.24\textwidth}
    \centering
    \includegraphics[width=\textwidth]{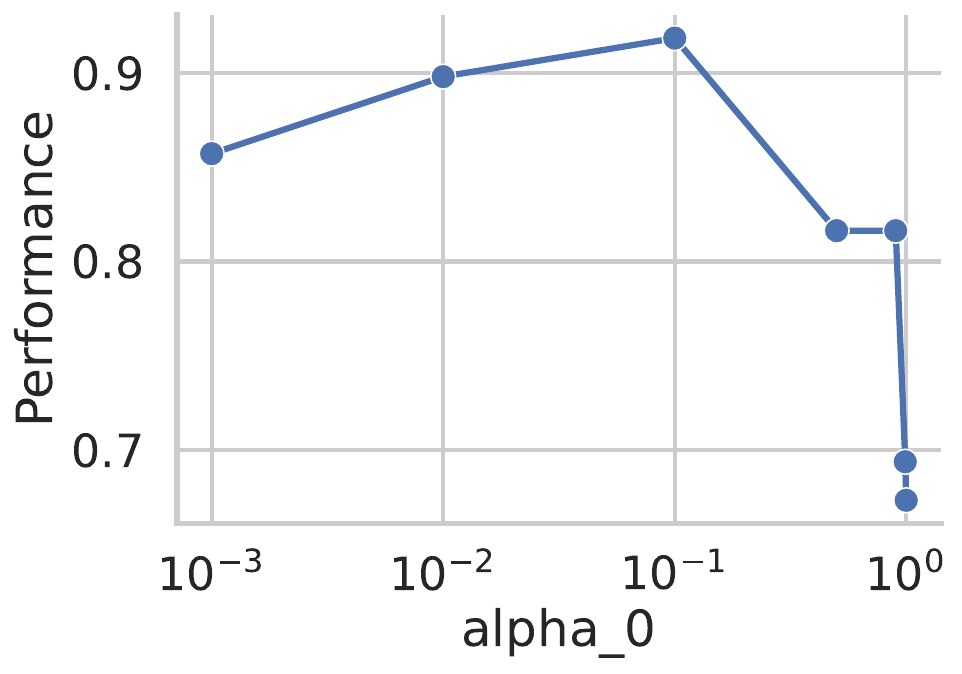}
    \vspace{-0.55cm}
    \caption{\EMUser}
    \vspace{0.0cm}
  \end{subfigure}
  \vspace{-0.09cm}
  \caption{Performance of WLKS-$\{0,D\}$ by the coefficient $\alpha_0$ in $\alpha_{0} \mK_{\WLS}^{0} + (1 - \alpha_{0}) \mK_{\WLS}^{D}$.}
  \label{fig:sensitivity_alpha_0}
\end{figure*}

\paragraph{Sensitivity Analysis of Hyperparameters}

Figures~\ref{fig:sensitivity_num_iters} and \ref{fig:sensitivity_alpha_0} demonstrate the performance sensitivity of the WLKS-$\{0,D\}$ with respect to the number of iterations $T$ and the kernel mixing coefficient $\alpha_0$. The best performance is achieved at iterations of $T = 2$ or $T = 3$, beyond which the WL coloring stabilizes and no further improvement is observed. For $\alpha_0$, the WLKS-$\{0,D\}$ is best-performed between $10^{-3}$ and $10^{-1}$, while performance drops sharply as $\alpha_0$ approaches 1. Since $\alpha_D = 1 - \alpha_0$ is larger than $\alpha_0$ in this range, this suggests that the subgraph labels rely more on global structures ($k = D$) than internal ones ($k=0$) for these datasets.

\section{Conclusion}
We proposed WLKS, a simple but powerful model for subgraph-level tasks that generalizes the Weisfeiler-Lehman (WL) kernel on induced $k$-hop neighborhoods. WLKS can enhance expressiveness by linearly combining kernel matrices from multiple $k$-hop levels, capturing richer structural information without redundant neighborhood sampling. Through extensive experiments on eight real-world and synthetic benchmarks, WLKS outperformed state-of-the-art GNN-based models on five datasets with reduced training times—ranging from $\times 0.01$ to $\times 0.25$ compared to existing models. Furthermore, WLKS does not need pre-computation, pre-training, GPUs, or extensive hyperparameter tuning.

Our method offers a promising and accessible alternative to GNN-based approaches for subgraph representation learning, but some tasks can still benefit from incorporating continuous features. We leave as future work the seamless integration of WLKS with Graph Neural Networks to leverage the expressive power of both structures and features.

\subsubsection*{Acknowledgments}
This work was supported by Institute of Information \& communications Technology Planning \& Evaluation (IITP) grant funded by the Korea government (MSIT) (No.RS-2022-II220184, Development and Study of AI Technologies to Inexpensively Conform to Evolving Policy on Ethics)

\bibliography{iclr2025_conference}
\bibliographystyle{iclr2025_conference}

\clearpage
\appendix
\section{Formal Comparison with Representative Related Work}\label{appendix:comparison}

In this section, we compare WLKS with highly related prior work, including Subgraph Neural Network~\citep[SubGNN;][]{alsentzer2020subgraph}, GNN with LAbeling trickS for Subgraph~\citep[GLASS;][]{wang2022glass}, Subgraph-To-Node Translation~\citep[S2N;][]{kim2024translating}, and GNN As Kernel~\citep[GNN-AK;][]{zhao2022from}.

While SubGNN employs message-passing within subgraphs, its reliance on ad hoc patch sampling and its separation of hand-crafted channels (e.g., position, neighborhood, structure) introduces complexity and potential sub-optimality in information aggregation. Without requiring hand-crafted patch designs or sampling strategies, WLKS captures a unified and expressive structural representation based on the theoretical rationale that structures at multiple levels are important.

GLASS uses separate message-passing for node labels that distinguish internal and global structures. This can enhance expressiveness by mixing representations from local and global structures similar to WLKS. However, GLASS has a limited ability to handle multiple labels in batched subgraphs; thus, a small batch size is required for GLASS. WLKS provides a generalized framework to represent fine-grained levels of structures around subgraphs, which can process multiple subgraphs efficiently by leveraging kernel methods.

S2N efficiently learns subgraph representations by compressing the global graph. However, this compression results in a loss of structural information and expressiveness in tasks where the global structure is important. In particular, since the approximation bound of S2N depends on how many subgraphs are spread out in the global graph, we cannot always expect a robust approximation. In contrast, WLKS does not rely on lossy compression and can yield informative representations using efficient kernel methods.

GNN-AK generates a graph representation by aggregating the information of each node's locally induced encompassing subgraph. Although there are local and global interactions, there are fundamental differences between WLKS. First, GNN-AK is designed for graph-level tasks, so the interactions between the graph itself and its local subgraphs are modeled. However, dealing with subgraph-level tasks is more challenging since modeling both the inside and the outside of the subgraph is required. WLKS encodes them by using multiple $k$-hop kernels. Second, GNN-AK has a large complexity that depends on the total number of neighbors and the sum of edges between neighbors, so it cannot be applied to a large global graph, unlike WLKS. In fact, the average number of nodes covered by the GNN-AK paper is much smaller, ranging from 25 to 430. In these perspectives, our study takes a complementary approach to GNN-AK, addressing aspects not covered in their work.

\section{Discussion on Relations between WL Isomorphism Test and WL Kernels}\label{appendix:wl_test_and_kernel}

In this section, we provide a detailed discussion to clarify the distinctions between Weisfeiler-Lehman (WL) isomorphism test and WL kernels, elaborating on how our work builds upon these foundational concepts.

The WL algorithm is recognized for testing graph isomorphism by iteratively refining node labels to capture the structural similarity between graphs. Its ability to distinguish non-isomorphic graphs is often considered a benchmark for evaluating the expressiveness of graph representation methods. While isomorphism distinguishability is theoretically significant, it can be overly restrictive in practical applications where the exact topological equivalence of graphs is not the primary concern. For example, many real-world tasks involve identifying structural similarities between graphs that may not be strictly isomorphic but share functional or semantic similarities.

Our work aligns more closely with the WL kernel framework~\citep{shervashidze2009fast,shervashidze2011weisfeiler}, which extends the application of the WL algorithm beyond isomorphism testing. WL kernels compute graph similarity based on histogram representations of subtree patterns generated by the WL algorithm. These histograms serve as compact summaries of graph structure, allowing for the comparison of graphs even when they are not isomorphic. In this context, WL kernels prioritize capturing similarities between graphs over distinguishing isomorphic structures. This broader perspective makes WL kernels particularly suitable for various graph-structured data, where the goal is to quantify structural resemblance rather than to test for isomorphism.

Building upon the WL kernel framework, our work introduces the WLKS method, which leverages WL histograms as measures of subgraph similarity. The key insight here is that WL histograms provide a rich representation of subtree patterns within graphs, enabling a nuanced comparison of internal and external structures of subgraphs.

\section{Step-by-Step Visualization of $\WLS^k$ algorithm}\label{appendix:model_steps}

In Figure~\ref{fig:model_steps}, we visualize each iteration of the $\WLS^k$ algorithm using an example of Figure~\ref{fig:model}.

\begin{figure}
  \centering
  \includegraphics[width=\textwidth]{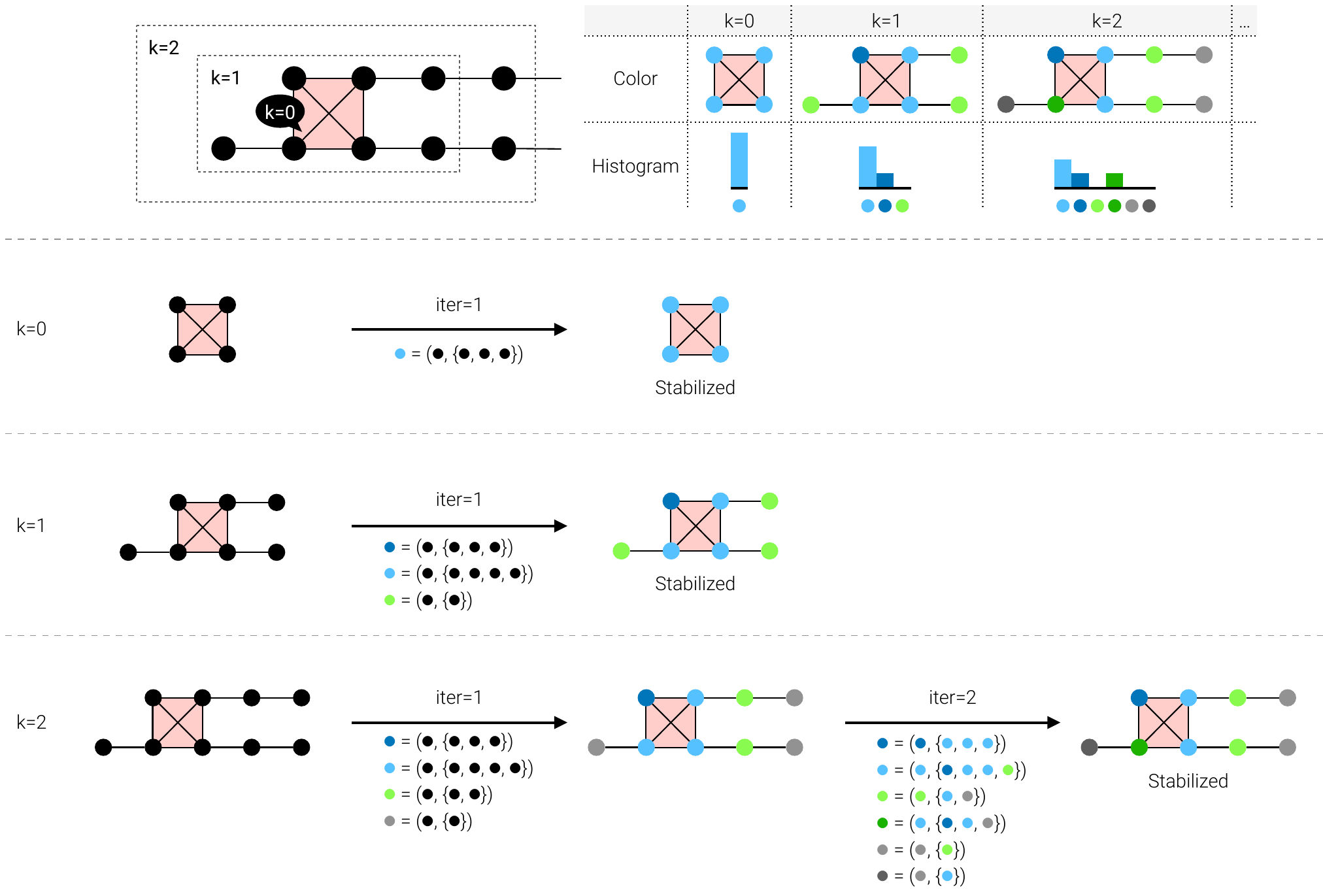}
  \vspace{-0.25cm}
  \caption{A step-by-step visualization of $\WLS^k$ algorithm (Algorithm~\ref{alg:wlks}) for $k \in \{ 0, 1, 2 \}$ using an example in Figure~\ref{fig:model}.}
  \vspace{-0.1cm}
  \label{fig:model_steps}
\end{figure}

\section{Using Kernels of Distance or Structural Encoding}\label{appendix:se}

It is well-known that additional structural features (often called labeling tricks, distance encoding, or structural encoding) can enhance the expressiveness of message-passing mechanisms under certain conditions~\citep{zhang2018link,li2020distance,zhang2021labeling,dwivedi2022graph,wang2022glass}. 

We argue that these approaches can have different effectiveness for subgraph-level tasks and kernel-based methods:
\begin{itemize}
    \item Zero-one labeling~\citep{zhang2021labeling,wang2022glass}: This binary labeling (assigning 0 to internal nodes and 1 to external nodes) shows limited expressiveness when aggregating labels to histograms as kernel inputs. Its histogram is represented as a length-2 vector (0 or 1), which only counts the number of nodes inside and outside the subgraph, thereby omitting finer structural details.
    \item SEAL's Double-radius node labeling~\citep{zhang2018link}: SEAL computes distances with respect to target structures (e.g., links) and can be applicable to $k$-hop neighborhoods of subgraphs but computationally challenging. While efficient for link prediction tasks due to the smaller size of enclosing subgraphs, extending this approach to general subgraphs becomes infeasible due to the computational overhead of calculating all pairwise distances.
    \item Distance Encoding (DE)~\citep{li2020distance} and Random Walk Structural Encoding (RWSE)~\citep{dwivedi2022graph}: DE uses landing probabilities of random walks from nodes in the node set to a given node, and RWSE uses diagonal elements of random walk matrices. In this line of work, random walk matrices are shown to encode structures in an expressive way, even on large-scale graphs.
\end{itemize}

We linearly combine \WLKSZD{} with an inner product kernel of RWSE (the walk length of 1 -- 64) sum-aggregated per subgraph, yielding a significant performance boost on the \CutRatiob dataset (from 60.0 to 96.0). However, this improvement is not shown across the other seven benchmarks we tested. We leave investigating which specific node labeling methods are effective, which aggregations of node labels are effective, and which kernels (e.g., linear, polynomial, RBF) best complement the specific node labeling as future work.

\end{document}